\documentclass[journal]{IEEEtai}

\usepackage[colorlinks,
urlcolor=red,
linkcolor=black,
citecolor=black]{hyperref}

\usepackage{color,array}

\usepackage{graphicx}

\usepackage{amsmath}
\usepackage{amssymb}
\usepackage{mathtools}
\usepackage{amsthm}
\usepackage{bm}
\usepackage{cite}
\usepackage{subfigure}
\usepackage[ruled,vlined]{algorithm2e} 

\newcommand{\upmodels}{{\perp \!\!\! \perp}}
\newcommand{\nupmodels}{{\not \! \perp \!\!\! \perp}}

\usepackage{soul}
\usepackage{color, xcolor}
\definecolor{lightblue}{RGB}{0, 191, 255}
\sethlcolor{lightblue}
\soulregister\cite7
\soulregister\eqref7
\soulregister\ref7

\newtheorem{lemma}{Lemma}
\newtheorem{proposition}{Proposition}
\newtheorem{definition}{Definition}
\newtheorem{corollary}{Corollary}

\begin{document}

\title{Identifying Unique Spatial-Temporal Bayesian Network without Markov Equivalence}

\author{Mingyu Kang,
	Duxin Chen,
	Ning Meng,
	Gang Yan,
	and Wenwu Yu,~\IEEEmembership{Senior~Member,~IEEE} 
	
	\thanks{This work is supported by the National Key R\&D Program of China under Grant No. 2022ZD0120003, the Zhishan Youth Scholar Program, the National Natural Science Foundation of China under Grant Nos. 62233004, 62273090, 62073076, and the Jiangsu Provincial Scientific Research Center of Applied Mathematics under Grant No. BK20233002. (corresponding authors: Duxin Chen, Wenwu Yu)}
	\thanks{Mingyu Kang is with the School of Cyber Science and Engineering, Southeast University, Nanjing 210096, China. (e-mail: kangmingyu@seu.edu.cn)}
	\thanks{Duxin Chen is with the School of Mathematics, Southeast University, Nanjing 210096, China. (e-mail: chendx@seu.edu.cn)}
	\thanks{Ning Meng is with the School of Cyber Science and Engineering, Southeast University, Nanjing 210096, China, and also with the R\&D department of Volkswagen China Technology Company (VCTC), Hefei 230601, China. (e-mail: ning.meng@volkswagen-anhui.com)}
	\thanks{Gang Yan is with the School of Physics Science and Engineering, Tongji University, Shanghai 200092, China. (e-mail: gyan@tongji.edu.cn).}
	\thanks{Wenwu Yu is with the Frontiers Science Center for Mobile Information Communication and Security, School of Mathematics, Southeast University, Nanjing 210096, China, and also with the Purple Mountain Laboratories, Nanjing 211102, China. (e-mail: wwyu@seu.edu.cn).}
}

%

\markboth{IEEE Transactions on Artificial Intelligence, DOI: 10.1109/TAI.2024.3483188}
{Author, \MakeLowercase{\textit{et al.}}: Bare Demo of IEEEtai.cls for IEEE Journals of IEEE Transactions on Artificial Intelligence}

\maketitle

\begin{abstract}
	
	Identifying vanilla Bayesian network to model spatial-temporal causality can be a critical yet challenging task. Different Markovian-equivalent directed acyclic graphs would be identified if the identifiability is not satisfied. To address this issue, Directed Cyclic Graph~\cite{DCG2013} is proposed to drop the directed acyclic constraint. But it does not always hold, and cannot model dynamical time-series process. Then, Full Time Graph~\cite{PGM2017} is proposed with introducing high-order time delay. Full Time Graph has no Markov equivalence class by assuming no instantaneous effects. But, it also assumes that the causality is invariant with varying time, that is not always satisfied in the spatio-temporal scenarios. Thus, in this work, a Spatial-Temporal Bayesian Network (STBN) is proposed to theoretically model the spatial-temporal causality from the perspective of information transfer. STBN explains the disappearance of network structure $X\rightarrow Z \rightarrow Y$ and $X\leftarrow Z \leftarrow Y$ by the principle of information path blocking. And finally, the uniqueness of STBN is proved. Based on this, a High-order Causal Entropy (HCE) algorithm is also proposed to uniquely identify STBN under time complexity $\mathcal{O}(n^3\tau_{max})$, where $n$ is the number of variables and $\tau_{max}$ is the maximum time delay. Numerical experiments are conducted with comparison to other baseline algorithms. The results show that HCE algorithm obtains state-of-the-art identification accuracy. The code is available at \url{https://github.com/KMY-SEU/HCE}.   
	
\end{abstract}

\begin{IEEEImpStatement}
This work presents work whose goal is to advance the field of Artificial Intelligence. This work provides a new way to think in causality though the Spatial-Temporal Bayesian Network, and proposes a High-order Causal Entropy algorithm to identify the network structure uniquely with polynomial time complexity. The model and algorithm are efficient and explainable to be applied to discover the causal interactions with high-order time delays in collective behaviors (e.g., animal flocks and crowds), industrial system (e.g., smart grid, intelligent transportation system and wireless internet), etc. 
\end{IEEEImpStatement}


\section{Introduction}

Identifying spatial-temporal causality can be a critical yet challenging task in many data-intensive scenarios, including collective behaviors~\cite{collective2016}, water systems~\cite{hydrology2021}, brain networks~\cite{DDAGs2024NeuroImage}, transportation systems~\cite{LV2024TITS} and electricity networks~\cite{cGAN2024Kang}, etc. A common point is that they can be modeled as a multivariate nonlinear system on the spatio-temporal interactions between variables, and the interactions usually have time delays.

Many algorithms~\cite{GC1969, CCM2012, OCE2015, FullCI2018, PCMCI2019, dynotears2020PRML, TimeByesian2021, TimeCausality2022ICLR, gong2023rhino, lippe2023causal} have been proposed to identify the spatio-temporal causality, that is modeled as a (causal) Bayesian network with a constraint of directed acyclic graph (DAG). However, the DAG can be hard or impossible to be identified if the ``structure identifiability''~\cite{PGM2017} is not satisfied. On the one hand, the independence test cannot distinguish the fork structure $X\leftarrow Z \rightarrow Y$ from the chain structures $X\rightarrow Z \rightarrow Y$ and $X\leftarrow Z \leftarrow Y$, if any three variables $X, Y, Z$ are given~\cite{equivalence2022}. As shown in Fig.~\ref{meq}(a), there are only four fundamental structures in a DAG. But $X\leftarrow Z \rightarrow Y$, $X\rightarrow Z \rightarrow Y$ and $X\leftarrow Z \leftarrow Y$ cannot be distinguished, because they conform to the same conditional independence $X\upmodels Y | Z$. Then, the forward and reverse directions are both detected, but only one is true. Thus, a class of Markov-equivalent DAGs would be detected finally. On the other hand, the DAG representation does not match the functional mechanism in data generation, especially in time-series process. E.g., as shown in Fig.~\ref{meq}(b), the two variables are interactive from the view of functional mechanism, but the DAG only allows one of the causal links. In this case, the causality would also be non-identifiable, and different DAGs would be identified as Markov equivalence class. 

\begin{figure}[htbp]
	\centering  
	\subfigure[]{
		\begin{minipage}[t]{0.85\linewidth}
			\centering
			\includegraphics[width=\linewidth]{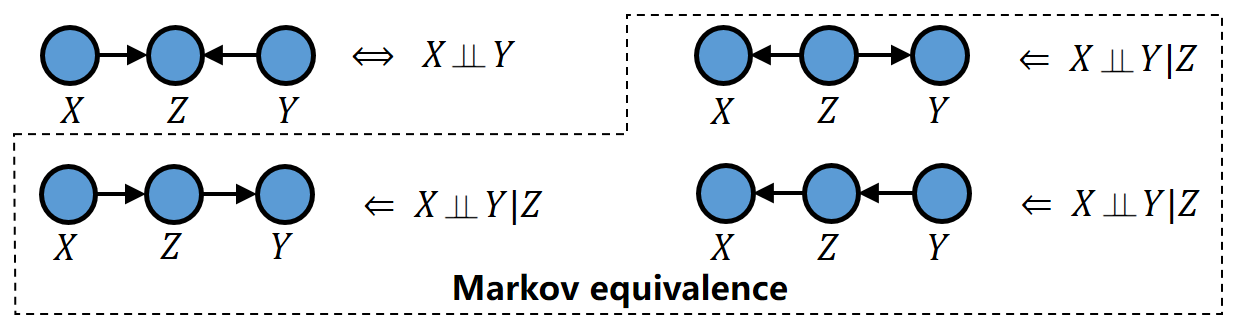}
		\end{minipage}
	}
	\subfigure[]{
		\begin{minipage}[t]{0.75\linewidth}
			\centering
			\includegraphics[width=\linewidth]{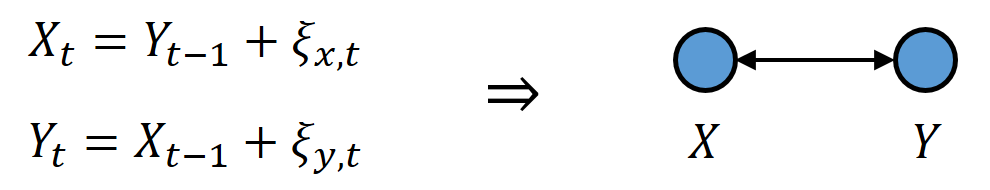}
		\end{minipage}
	}
	\caption{Markov equivalence. Here, $X_t$ and $Y_t$ represent the variables $X$ and $Y$ at time $t$. $\xi$ represents noise variable.}
	\label{meq}
\end{figure}

To address this issue, some special networks are designed, such as Directed Cyclic Graph~\cite{DCG2013}. Like the DAG-based Bayesian network, Directed Cyclic Graph can also be used to build a Structural Causal Model without further assumptions~\cite{equivalence2021}. However, Directed Cyclic Graph does not always hold because the common cause principle~\cite{causality2009} is not always satisfied. Moreover, that network does not consider time delay, so that cannot model dynamical time-series process. Thus, Full Time Graph~\cite{PGM2017}, a temporal Bayesian network, is proposed with introducing high-order time delay, and it has no Markov equivalence class if without instantaneous effects (see Theorem~10.1 in~\cite{PGM2017}). However, Full Time Graph assumes that the causality is invariant with varying time, that is not always satisfied in the spatio-temporal scenarios. 


Thus, in this work, a Spatial-Temporal Bayesian Network (STBN) is proposed to theoretically model the spatial-temporal causality from the perspective of information transfer. Transfer entropy~\cite{TE2000} is introduced here to quantify the information transfer. Based on this, the equivalence is proved between zero transfer entropy, time-series unpredictability and d-separation (equal to conditional independence). Then, information path blocking explains the disappearance of $X\rightarrow Z \rightarrow Y$ and $X\leftarrow Z \leftarrow Y$, and this further induces the uniqueness of STBN. Thus, the corresponding algorithm can be designed to identify the unique network structure of STBN by the transfer entropy with information path blocking. We call it as high-order causal entropy (HCE) algorithm here.

Thus, the main contributions are summarized as follows:
\begin{enumerate}
	\item[1.] STBN is proposed to address the issue of Markov equivalence.  
	
	\item[2.] HCE algorithm is proposed to identify the unique network structure of STBN.
	
	\item[3.] Sufficient experiments are conducted to support the conclusion of state-of-the-art accuracy of HCE algorithm on nonlinear data with time-varying distribution shift, compared to the other widely-used baseline algorithms. 

\end{enumerate}

\section{Background and Notation}


Throughout the paper, the following syntactic conventions are used. An upper case letter (e.g., $X_i, Y$) is denoted as a variable, and the same letter in lower case (e.g., $x_i, y$) is denoted as the state value of that variable. A bold-face capitalized letter (e.g., $\mathbf{X}_i, \mathbf{Y}$) is denoted as a set of variables, and a corresponding bold-face lower-case letter (e.g., $\mathbf{x}_i, \mathbf{y}$) is denoted as the assigned state values to each variable in that set. 

Then, a general nonlinear system is considered. It is represented as a coupling of multiple stochastic variables $\mathbf{V}=\{X_1, \dots, X_n\}$ as follow:
\begin{equation}
	X_{i,t} = f_i(X_{j,t-\tau}, j = 1, \dots, n, \tau=1, \dots, \tau_{max}) + \xi_{i, t}, 
	\label{sys}
\end{equation}
for each $X_i \in \mathbf{V}$ at time $t$. Here, $\tau_{max}$ is the maximum time delay, and $\xi_{i, t}$ is additive noise. Note that, all our works are built on this system.

\subsection{Bayesian Network and Markov Equivalence}


Bayesian network can be built on $\mathbf{V}=\{X_1, \dots, X_n\}$ in Eq.~(\ref{sys}). It is a DAG $\mathcal{G}=(\mathbf{V}, \mathbf{E})$ that represents probabilistic knowledge~\cite{Aprimer2016}, where $\mathbf{E}$ is the set of directed edges. Note that, if the Structural Causal Model is mentioned, the Bayesian network is causal, because every Structural Causal Model is associated with a graphical causal model, and the Bayesian network is one realization of the graphical causal models~\cite{Aprimer2016}. In Bayesian network, DAG is a type of graphical language to describe all of conditional independence in joint distribution. It factorizes the joint distribution into a product of multiple conditional probability distributions as follow:
\begin{equation}
	P(X_1, \dots, X_n) = \prod_{i=1}^n P(X_i|\mathbf{Pa}_i), 
\end{equation}
where $\mathbf{Pa}_i$ is the direct predecessors of $X_i$ in $\mathcal{G}$, that is also called as causal parent. Thus, with the Bayesian network, conditional independence is satisfied, that is, $X_i \upmodels \{X_1, X_2, \dots, X_{i-1}\} \backslash \mathbf{Pa}_i | \mathbf{Pa}_i$.

To identify the conditional independence from observation, d-separation provided a graphical criterion, that is defined as follow:
\begin{definition}[d-separation~\cite{PGM2009}]
	Let $\mathbf{X}$, $\mathbf{Y}$ and $\mathbf{Z}$ be three disjoint subsets of $\mathbf{V}$, and let $\mathbf{p}$ be any path from a node in $\mathbf{X}$ to a node in $\mathbf{Y}$ regardless of direction. $\mathbf{Z}$ is said to block $\mathbf{p}$ if and only if there is a node $v\in \mathbf{p}$ satisfying one of the following items.
	\begin{enumerate}
		\item $v$ has v-structure (two nodes $a, b\in \mathbf{p}$ pointing to $v$, namely $a\rightarrow v\leftarrow b$), and neither $v$ nor its any descendants are in $\mathbf{Z}$; 
		\item $v$ in $\mathbf{Z}$ and $v$ does not have v-structure.  
	\end{enumerate} 
	Then, $\mathbf{Z}$ d-separate $\mathbf{X}$ and $\mathbf{Y}$, denoted as $\mathbf{X} \upmodels_\mathcal{G} \mathbf{Y} | \mathbf{Z}$, if $\mathbf{Z}$ blocks any $\mathbf{p}$.  
\end{definition}

Based on d-separation, the assumptions of causal Markov and faithfulness can be defined as follows:
\begin{definition}[Causal Markov~\cite{PGM2009}]
	The joint distribution $P$ is Markovian to the DAG $\mathcal{G}$ if 
	\begin{equation}
		\mathbf{X} \upmodels_\mathcal{G} \mathbf{Y} | \mathbf{Z} \Rightarrow \mathbf{X} \upmodels \mathbf{Y} | \mathbf{Z},
	\end{equation}
	where $\mathbf{X}$, $\mathbf{Y}$ and $\mathbf{Z}$ are three disjoint subsets of $\mathbf{V}$. 
	\label{mp}
\end{definition}
\begin{definition}[Faithfulness~\cite{PGM2009}]
	Probability distribution $P$ is faithful to the DAG $\mathcal{G}$ if 
	\begin{equation}
		\mathbf{X} \upmodels \mathbf{Y} | \mathbf{Z} \Rightarrow \mathbf{X} \upmodels_\mathcal{G} \mathbf{Y} | \mathbf{Z}
	\end{equation}
	for all disjoint subsets $\mathbf{X}$, $\mathbf{Y}$ and $\mathbf{Z}$.
	\label{ff}
\end{definition}
The two assumptions are also used to identify DAG-based network structure from observation directly by independence test. However, it is often not easy because of Markov equivalence. As shown in Fig.~\ref{meq}(a), there are only four structures in a Bayesian network if given any three variables $X, Y, Z$. If $X\upmodels Y|Z$ is detected, the three equivalent structures are all corresponding to it. Thus, the identified Bayesian network are often a group of DAGs conforming to the Markov equivalence.

The issue of Markov equivalence can induce many problems that restrict the usage. Firstly, Markov blanket is not equal to the $\mathbf{Pa}_i$. The Markov blanket is defined as
\begin{definition}[Markov blanket~\cite{PGM2009}]
	For a DAG $\mathcal{G}$ and a target variable $X$, if there is a variable set $\mathbf{M} \subset \mathbf{V}$ satisfied 
	\begin{equation}
		X \upmodels_\mathcal{G} \mathbf{V} \backslash (\mathbf{M} \cup \{X\}) | \mathbf{M},
	\end{equation}
	then the minimal $\mathbf{M}$ is a Markov blanket, which contains all knowledge for predicting $X$. 
	\label{mb}
\end{definition}
If Markov equivalence class exists, the causal direction is hard to identified truly, because forward and reverse directions may be both detected, but only one is true in DAG. Thus, the children may be used to predict parents. That is, $\mathbf{Pa}_i \subseteq \mathbf{M}$ as defined, which is usually not what we want.

Secondly, it is hard to infer counterfactuals. Based on Pearl's causal ladder~\cite{Aprimer2016}, inferring counterfactual has three steps: identify a (causal) Bayesian network, intervene the network, and infer counterfactuals. However, if the network structure is not unique, the inferred counterfactual would not be unique. This is usually unacceptable.

\subsection{Predictability and Information Transfer}



Predictability was first proposed in Granger Causality (GC) test~\cite{GC1969} to detect and quantify causality in time-series linear system by vector autoregressive model. The predictability to variable $X$ provided by variable $Y$, $\mathcal{F}_{Y\rightarrow X}$, is equal to the logarithmic variance ratio between with and without $Y$ in Gaussian vector autoregression process~\cite{geweke1982, geweke1984}. That is,
\begin{equation}
	\begin{aligned}
		&\mathcal{F}_{Y\rightarrow X} = \log\left[\frac{Var(\xi_{t})}{Var(\xi_{t}')}\right], \\
		&\xi_{t} = X_{t} - \sum_{\tau=1}^{\tau_{max}} X_{t-\tau}w_{t-\tau}, \\
		&\xi_{t}' = X_{t} - \sum_{\tau=1}^{\tau_{max}} X_{t-\tau}w_{t-\tau} - \sum_{\tau=1}^{\tau_{max}} Y_{t-\tau}u_{t-\tau},
	\end{aligned}
\end{equation}
where $Var(\cdot)$ denotes the variance for samples as $t=\tau_{max}+1, \dots, k$ if $k$ samples is given. $w$ and $u$ are the numerical weights of $X$ and $Y$, respectively. Thus, $\mathcal{F}_{Y\rightarrow X}$ in GC is the reduction of uncertainty from $Y$ to $X$ actually.

The nonlinear GC can be sophisticated, because of the nonlinearity in multivariate system Eq.~(\ref{sys}). On the one hand, nonlinear autoregressive model is hard to build. On the other hand, it is time consuming to pair-wisely detect causality between each pair of variables. To address this issue, Barnett, et al.~\cite{TE2012} proposed an approach to quantify bi-variate nonlinear GC based on $\tau_{max}$-lag transfer entropy $\mathcal{T}_{Y\rightarrow X}$. It is defined as
\begin{equation}
	\begin{aligned}
		&\mathcal{T}_{Y\rightarrow X} \\
		=& H(X_t|X_{t-1}^{(\tau_{max})}) - H(X_t|X_{t-1}^{(\tau_{max})}, Y_{t-1}^{(\tau_{max})}) \\
		=& -\frac{1}{k-\tau_{max}} \sum_{t=\tau_{max}+1}^{k}\log \frac{f(x_{t-1}^{(\tau_{max})}; \bm{\theta}_0)}{f(x_{t-1}^{(\tau_{max})}, y_{t-1}^{(\tau_{max})}; \bm{\theta})} \\
		=& -\frac{1}{k-\tau_{max}} \log \frac{\mathcal{L}(\hat{\bm{\theta}}_0|X, Y)}{\mathcal{L}(\hat{\bm{\theta}}|X, Y)} \\
		=& \mathcal{F}_{Y\rightarrow X}.
	\end{aligned}
\end{equation} 
Here, $X_{t-1}^{(\tau_{max})}=\{X_{t-1}, \dots, X_{t-\tau_{max}}\}$, and $x_{t-1}^{(\tau_{max})}$ is the sampled state value. $Y_{t-1}^{(\tau_{max})}$ and $y_{t-1}^{(\tau_{max})}$ is the same as this. Moreover, $H(Y|X)=-\sum_{x, y}P(x, y)\log P(y|x)$, and $f(\cdot|\bm{\theta})$ is nonlinear function with parameter $\bm{\theta}$. $\mathcal{L}(\cdot |X, Y)$ is the probabilistic likelihood of $X$ and $Y$. $\bm{\theta}\in \mathbf{\Theta}$, and $\mathbf{\Theta}$ is full parameter set. $\bm{\theta}_0\in \mathbf{\Theta}_0 = \{\bm{\theta}\in \mathbf{\Theta}|f(x_{t-1}^{(\tau_{max})}, y_{t-1}^{(\tau_{max})}; \bm{\theta}) \text{ does not depend on } y_{t-1}^{(\tau_{max})}\}$. $\hat{\bm{\theta}}_0$ and $\hat{\bm{\theta}}$ are the corresponding estimation, respectively. $\mathcal{T}_{Y\rightarrow X}$ quantify the reduction of uncertainty and the information transfer from $Y$ to $X$ by the log-likelihood ratio.

\section{Spatial-Temporal Bayesian Network}


In this work, STBN is proposed as a type of Bayesian network, that represents spatial-temporal causality in interactions between variables with high-order time delays. The proofs for the uniqueness of STBN are also provided. Moreover, HCE algorithm is proposed to uniquely identify the network structure of STBN by measuring high-order causal entropy.

\subsection{Spatial-Temporal Causality}


Spatial-temporal causality is defined here as the probabilistic knowledge describing the causal interactions between variables with high-order time delays during a dynamical process, e.g., in Eq.~(\ref{sys}). Consider an STBN $\mathcal{G}=(\mathbf{V}, \mathbf{E})$ built on Eq.~(\ref{sys}). $\mathbf{V}=\{X_1, \dots, X_n\}$ is the variable set, and $\mathbf{E}$ is the edge set. Note that, due to the introducing of high-order time delay, thus, the nodes in STBN actually denote variables at some times. For example, $X_{i, t}$ and $X_{i, t-1}$ are both variable $X_i$ in traditional Bayesian network, but they denote $X_i$ at different times here, that is, they are slightly different. To define spatial-temporal causality, two assumptions are firstly defined as:
\begin{enumerate}
	\item[1.] \textbf{Temporal assumption:} the cause precedes the effect.
	\item[2.] \textbf{Causal assumption:} the state of effect is predictable by the state of its direct cause. 
\end{enumerate}
The two assumptions are fundamental enough, and widely accepted in~\cite{OCE2015, spirtes2016, PGM2017, CCM2012, PCMCI2019, TPAMI6420843, TII10352645}. 

\begin{figure}[htbp]
	\centering
	\includegraphics[width=\linewidth]{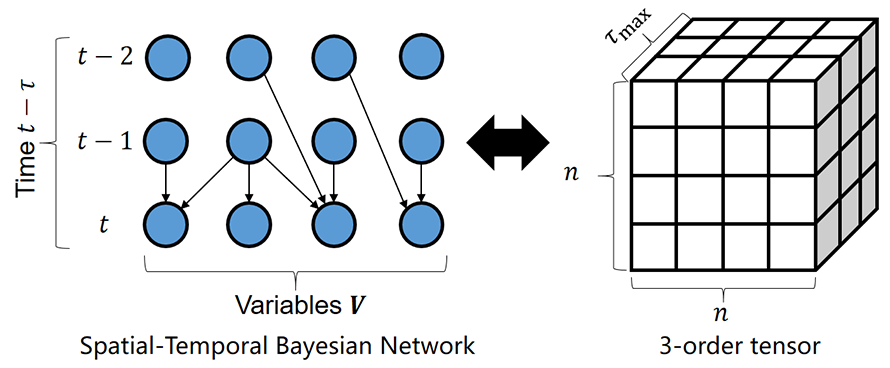}
	\caption{Diagram of Spatial-Temporal Bayesian Network.}
	\label{dag}
\end{figure}

Thus, the adjacency of $\mathcal{G}$ can be denoted as a real 3-order tensor $A \in \mathbb{R}^{n \times n \times \tau_{max}}$, as shown in Fig.~\ref{dag}. The weight $A_{ji\tau}$ denotes the causal strength of $X_{j,t-\tau} \rightarrow X_{i,t}$ with time lag $\tau$. Intuitively, STBN describes the stable multivariate causal relationship during a period of time, or the momentary causality at time $t$ in time series. Then, 
\begin{proposition}[Acyclic]
	STBN $\mathcal{G}$ is a DAG.
	\label{acyc}
\end{proposition}
\begin{proof}
	If $\mathcal{G}$ is not a DAG, there is at least one path $\mathbf{p}$ which is a cycle. Then, for any two different nodes $X, Y \in \mathbf{p}$, there must be two different path $X\rightarrow \dots \rightarrow Y$ and $Y\rightarrow \dots \rightarrow X$ without intersection except for them two. Then, $X$ is the direct or indirect cause of $Y$, and the direct or indirect effect of $Y$ at the same time. This is contradictory to the \textbf{temporal assumption}. 
\end{proof}
Thus, due to the \textbf{Proposition~\ref{acyc}}, the assumptions of causal Markov and faithfulness (see \textbf{Definition~\ref{mp}} and \textbf{Definition~\ref{ff}}) can also be applied to STBN. This enables to identify network structure from observation.

\subsection{Proofs for Uniqueness}


Here, transfer entropy~\cite{TE2000} is introduced to quantify information transfer. Firstly, the causal equivalence between zero transfer entropy, unpredictability and d-separation is proved. Based on this, it is proved that parent set $\mathbf{Pa}_X$ is equal to Markov blanket $\mathbf{M}_X$ for arbitrary variable $X$, due to information path blocking. Finally, the uniqueness of STBN is obtained. Thus, first of all, two lemmas are introduced to prove the causal equivalence as follows:    
\begin{lemma}
	For a variable $X$ and another variable $Y$ preceding $X$, $X \upmodels_\mathcal{G} Y | \mathbf{Z}$ if and only if the transfer entropy $\mathcal{T}_{Y\rightarrow X|\mathbf{Z}}=0$, where $\mathbf{Z}$ is the condition set of $X$. 
	\label{mcite}
\end{lemma}
\begin{proof}
	Given the variables $X$ and $Y$, then $X \upmodels_\mathcal{G} Y | \mathbf{Z} \Rightarrow X \upmodels Y | \mathbf{Z}$ because of causal Markov (see \textbf{Definition~\ref{mp}}). Thus, $\mathcal{T}_{Y\rightarrow X|\mathbf{Z}} = H(X|\mathbf{Z}) - H(X|Y, \mathbf{Z}) = 0$. On the other hand, if $\mathcal{T}_{Y\rightarrow X|\mathbf{Z}} = 0$, then $X \upmodels Y | \mathbf{Z}$. Thus, $X \upmodels Y | \mathbf{Z} \Rightarrow X {\upmodels}_\mathcal{G} Y | \mathbf{Z}$ because of faithfulness (see \textbf{Definiton~\ref{ff}}). 
\end{proof}
\begin{lemma}
	For a variable $X$ and another variable $Y$ preceding $X$, the transfer entropy $\mathcal{T}_{Y\rightarrow X|\mathbf{Z}}=0$ if and only if $\mathcal{F}_{Y\rightarrow X|\mathbf{Z}}=0$, where $\mathbf{Z}$ is a condition set of $X$.
	\label{tef}
\end{lemma}
\begin{proof} 
	If given variables $X$, $Y$ and $\mathbf{Z}$, their $k$ joint samples $(x^{(k)}, y^{(k)}, \mathbf{z}^{(k)})$ are collected from system Eq.~(\ref{sys}). That is, $x^{(k)} = \{x_{1}, \dots, x_k\}$, and the other two are the same with one-to-one temporal correspondence. They are all assumed to be ergodic, that indicates the dynamical system have visited all parts of the state space with varying time~\cite{ergodic2000}. This implies the expectation of functions (see Eq.~(\ref{sys})) on the observed samples would converge to the expectation on the state space. 
	
	Thus, transfer entropy can be defined as 
	\begin{equation}
		\mathcal{T}_{Y\rightarrow X|\mathbf{Z}} = H(X|\mathbf{Z}) - H(X|Y, \mathbf{Z}).
		\label{te}
	\end{equation}
	Then, the prediction model $f$ with parameters $\bm{\theta}$ can be defined as 
	\begin{equation}
		f(x|y, \mathbf{z}; \bm{\theta}) = P(X|Y, \mathbf{Z}; \bm{\theta}).
	\end{equation}
	Here, $P(\cdot|\cdot; \bm{\theta})$ is parameterized conditional probability distribution function. The model $f$ is assumed to be identifiable and the parameters are assumed to be unique such that $f(\cdot|\cdot; \bm{\theta}_1) \neq f(\cdot|\cdot; \bm{\theta}_2)$ for $\bm{\theta}_1 \neq \bm{\theta}_2$. Thus, the unique true parameter $\bm{\theta}^*$ satisfy 
	\begin{equation}
		f(x|y, \mathbf{z}; \bm{\theta}^*) = P(X|Y, \mathbf{Z}; \bm{\theta}^*) = P(X|Y, \mathbf{Z}).
		\label{pxyz}
	\end{equation}
	Thus, according to ergodic assumption, $x_t$ is only up to $y_t$ and $\mathbf{z}_t$ for any $t=1, \dots, k$ in $f$, the likelihood for parameters $\bm{\theta}$ can be written as 
	\begin{equation}
		\begin{aligned}	
			&\quad \mathcal{L}_k(\mathbf{\theta}|x^{(k)}, y^{(k)}, \mathbf{z}^{(k)}) \\
			=&P(x^{(k)}, y^{(k)}, \mathbf{z}^{(k)}; \bm{\theta}) \\
			=&P(x_k | x^{(k-1)}, y^{(k)}, \mathbf{z}^{(k)}; \bm{\theta}) \\
			&\times P(y_k, \mathbf{z}_k|x^{(k-1)}, y^{(k-1)}, \mathbf{z}^{(k-1)}; \bm{\theta}) \\
			&\times P(x^{(k-1)}, y^{(k-1)}, \mathbf{z}^{(k-1)}; \bm{\theta}) \\
			=&f(x_k | y_k, \mathbf{z}_k; \bm{\theta})q(y_k, \mathbf{z}_k)\mathcal{L}_{k-1} \\
			=&\prod_{t=1}^{k}f(x_t | y_t, \mathbf{z}_t; \bm{\theta})q(y_t, \mathbf{z}_t),
		\end{aligned}
	\end{equation} 
	where $P(x_t | x^{(t-1)}, y^{(t)}, \mathbf{z}^{(t)}; \bm{\theta})=f(x_t | x^{(t-1)}, y^{(t)}, \mathbf{z}^{(t)}; \bm{\theta})=f(x_t | y_t, \mathbf{z}_t; \bm{\theta})$. Moreover, the parameter set $\bm{\theta}$ only affects $X$, thus $P(y_t, \mathbf{z}_t|x^{(t-1)}, y^{(t-1)}, \mathbf{z}^{(t-1)}; \bm{\theta})$ can be simplified as a distribution $q(y_t, \mathbf{z}_t)$. $q(y_t, \mathbf{z}_t)$ does not equal to zero almost everywhere, and not refer to $\bm{\theta}$. Then, to maximize the likelihood $\mathcal{L}_k$, the average log-likelihood is obtained as
	\begin{equation}
		\begin{aligned}
			\ell(\bm{\theta}|x^{(k)}, y^{(k)}, \mathbf{z}^{(k)}) 
			&=\frac{1}{k}\log \mathcal{L}_k(\bm{\theta}|x^{(k)}, y^{(k)}, \mathbf{z}^{(k)}) \\
			&\sim\frac{1}{k}\sum_{t=1}^k\log f(x_t | y_t, \mathbf{z}_t; \bm{\theta}).
		\end{aligned}
	\end{equation}
	Thus, the unique parameter $\bm{\theta}^*$ for the prediction model $f$ can always be obtained as $k\rightarrow \infty$. Thus, according to Birkhoff-Khinchin ergodic theorem~\cite{ergodic2000} and Eq.~(\ref{pxyz}), it is obtained that
	\begin{equation}
		\ell \stackrel{a.s.}\longrightarrow \mathbb{E}[\log f(x|y, \mathbf{z}; \bm{\theta})],
		\label{as1}
	\end{equation}
	and then, 
	\begin{equation}
		\mathbb{E}[\log f(x|y, \mathbf{z}; \bm{\theta})]\stackrel{a.s.}\longrightarrow-H(X|Y, \mathbf{Z})
		\label{as2}
	\end{equation}
	as $k\rightarrow \infty$. Here, $\mathbb{E}[\cdot]$ is expectation. In practical experiments, this is to converge with small $k$.
	
	Then, a nested null model can be defined as $H_0: \bm{\theta}\in \bm{\Theta}_0$. $\bm{\Theta}_0 = \{\bm{\theta}\in \bm{\Theta}| X$ are indenpendent of $Y$ given $\mathbf{Z}\}$, and $\bm{\Theta}_0 \subseteq \bm{\Theta}$. $\bm{\Theta}$ is full parameter set. Then, the likelihood ratio can be defined as
	\begin{equation}
		\Lambda(x^{(k)}, y^{(k)}, \mathbf{z}^{(k)})	= \frac{\mathcal{L}_k(\hat{\bm{\theta}}_0|x^{(k)}, y^{(k)}, \mathbf{z}^{(k)})}{\mathcal{L}_k(\hat{\bm{\theta}}|x^{(k)}, y^{(k)}, \mathbf{z}^{(k)})},
	\end{equation}
	where $\hat{\bm{\theta}}_0$ and $\hat{\bm{\theta}}$ are maximum likelihood estimators for $\bm{\Theta}_0$ and $\bm{\Theta}$, respectively. It is intuitive that $\Lambda(x^{(k)}, y^{(k)}, \mathbf{z}^{(k)})$ measures the dependence degree of predicting $X$ given $Z$ to $Y$. Thus, the degree of predictability can be defined as
	\begin{equation}
		\mathcal{F}_{Y\rightarrow X|\mathbf{Z}} = -\frac{1}{k}\log \Lambda(x^{(k)}, y^{(k)}, \mathbf{z}^{(k)}),
	\end{equation}
	Where $\mathcal{F}_{Y\rightarrow X|\mathbf{Z}}\in [0,+\infty)$. Then, $\mathcal{F}_{Y\rightarrow X|\mathbf{Z}} \stackrel{a.s.}\longrightarrow \mathcal{T}_{Y\rightarrow X|\mathbf{Z}}$ as $k\rightarrow \infty$ obviously according to Eq.~(\ref{te}), Eq.~(\ref{as1}) and Eq.~(\ref{as2}). Thus, $\mathcal{T}_{Y\rightarrow X|\mathbf{Z}}=0$ if and only if $\mathcal{F}_{Y\rightarrow X|\mathbf{Z}}=0$.
\end{proof}
With the two lemmas, the causal equivalence can be stated as follows:
\begin{proposition}[Causal equivalence]
	In STBN $\mathcal{G}$ built on System Eq.~(\ref{sys}), zero transfer entropy, time-series unpredictability and d-separation in $\mathcal{G}$ are equivalent, that is, 
	\begin{equation}
		\mathcal{F}_{Y\rightarrow X|\mathbf{Z}} = 0 \iff \mathcal{T}_{Y\rightarrow X|\mathbf{Z}} = 0 \iff X \upmodels_\mathcal{G} Y | \mathbf{Z},
	\end{equation} 
	for any variable $Y$ preceding the target variable $X$, and $\mathbf{Z}$ is the set of some variables preceding $X$ as condition set. 
	\label{ceq}
\end{proposition}
In \textbf{Proposition~\ref{ceq}}, the transfer entropy bridges the predictability and the d-separation. This enables to understand the disappearance of $X\rightarrow Z \rightarrow Y$ and $X\leftarrow Z \leftarrow Y$ for arbitrary variables $X, Y, Z$. If $X\rightarrow Z \rightarrow Y$ is detected in STBN, then $X\upmodels_{\mathcal{G}} Y|Z \Rightarrow \mathcal{T}_{X\rightarrow Y|Z}=0 \Rightarrow \mathcal{F}_{X\rightarrow Y|Z}=0$. Thus, the information path is blocked, and $X$ provides no more predictability to $Y$ if $Z$ is detected at the same time. Thus, in $X\rightarrow Z \rightarrow Y$, the edge $X\rightarrow Z$ can be overlooked, and this is the same to do in $X\leftarrow Z \leftarrow Y$, because the minimal variable set to predict target variable is really what we want. 


Thus, with the bridge of causal equivalence, parent set $\mathbf{Pa}_X = \mathbf{M}_X$ for any target variable $X$ can be obtained in STBN. This is because all variables in $\mathbf{Pa}_X$ precede $X$, due to the \textbf{temporal assumption}, and as shown in Fig.~\ref{dag}. Note that, $\mathbf{M}_X$ is the Markov blanket of $X$ (see \textbf{Definition~\ref{mb}}), that is the minimal variable set for prediction. $\mathbf{Pa}_X = \mathbf{M}_X$ can be proved by the following proofs from \textbf{Proposition~\ref{more}} to \textbf{Corollary~\ref{paxmx}}.

\begin{proposition}
	$\forall Y \in \mathbf{Pa}_X$, the variable $Y$ satisfies $\mathcal{F}_{Y\rightarrow X|\mathbf{Pa}_X \backslash \{Y\}}>0$.
	\label{more}
\end{proposition}
\begin{proof}
	For any variable $Y \in \mathbf{Pa}_X$, no variable set can d-separate $X$ and $Y$. Thus, $\mathcal{F}_{Y\rightarrow X|\mathbf{Pa}_X \backslash \{Y\}}>0$ is obtained from $Y \nupmodels_\mathcal{G} X|\mathbf{Pa}_X \backslash \{Y\}$ according to \textbf{Proposition~\ref{ceq}} (i.e., $\mathcal{F}_{Y\rightarrow X| \mathbf{Pa}_X \backslash \{Y\}} = 0  \iff Y \upmodels_\mathcal{G} X | \mathbf{Pa}_X \backslash \{Y\}$). 
\end{proof}

Independence is used to search Markov blanket (see \textbf{Definition~\ref{mb}}) in traditional Bayesian network~\cite{PC2000}. But due to Markov equivalence, the identified Markov blanket $\mathbf{M}_X$ are larger than or equal to the true parent set $\mathbf{Pa}_X$. That is, spurious edges are to be identified inevitably. However, this can be addressed in STBN, because \textbf{temporal assumption} prevents causal information flow from effect to cause. Thus,
\begin{proposition}
	$\forall Y$ preceding $X$ and $Y\notin \mathbf{Pa}_X$, the variable $Y$ satisfies $\mathcal{F}_{Y\rightarrow X|\mathbf{Pa}_X}=0$.
	\label{eq0}
\end{proposition}
\begin{proof}
	In STBN $\mathcal{G}$, for any variable $Y$ preceding $X$ and $Y\notin \mathbf{Pa}_X$, $Y \upmodels_\mathcal{G} X|\mathbf{Pa}_X$ because all information paths from $Y$ to $X$ are blocked by $\mathbf{Pa}_X$. Thus, $Y \upmodels_\mathcal{G} X|\mathbf{Pa}_X \backslash \{Y\}$ because $\mathbf{Pa}_X = \mathbf{Pa}_X \backslash \{Y\}$. Thus, $\mathcal{F}_{Y\rightarrow X|\mathbf{Pa}_X \backslash \{Y\}}=0$ according to \textbf{Proposition~\ref{ceq}}. 
\end{proof}
Then, it is obtained that 
\begin{corollary}
	$\forall Y$ preceding $X$, $Y \notin \mathbf{Pa}_X$ if and only if $\mathcal{F}_{Y\rightarrow X|\mathbf{Pa}_X \backslash \{Y\}}=0$. 
	\label{inpa}
\end{corollary}
Thus, if $\mathbf{V}$ is the full variable set, then 
\begin{proposition}
	For any nested variable set $\mathbf{M} \supseteq \mathbf{Pa}_X$, and all variables in $\mathbf{M}$ precede $X$, then $\forall Y \in \mathbf{V}\backslash\mathbf{M} \cup \{X\}$, $Y \upmodels_\mathcal{G} X | \mathbf{M}$, that is, $\mathcal{F}_{Y\rightarrow X|\mathbf{M}}=0$. 
	\label{nestedM}
\end{proposition}
\begin{proof}
	Let $\mathbf{M} \supseteq \mathbf{Pa}_X$ and all variables in $\mathbf{M}$ precede $X$. For any selected $Y \in \mathbf{V}\backslash\mathbf{M} \cup \{X\}$, then $Y\notin \mathbf{Pa}_X$. According to \textbf{Proposition~\ref{eq0}}, $\mathcal{F}_{Y\rightarrow X|\mathbf{Pa}_X}=0$ is obtained. And then, $Y \upmodels_\mathcal{G} X | \mathbf{Pa}_X$ due to the causal equivalence. That means $\mathbf{Pa}_X$ blocks any information paths between $X$ and $Y$, as shown in Fig.~\ref{figure: prop5coro2}(a). Thus, $\mathbf{M}$ also blocks any information paths between $X$ and $Y$ because $\mathbf{M} \supseteq \mathbf{Pa}_X$. Thus, $Y \upmodels_\mathcal{G} X | \mathbf{M}$. 
\end{proof}

\begin{figure}[t!]
	\centering
	\includegraphics[width=\linewidth]{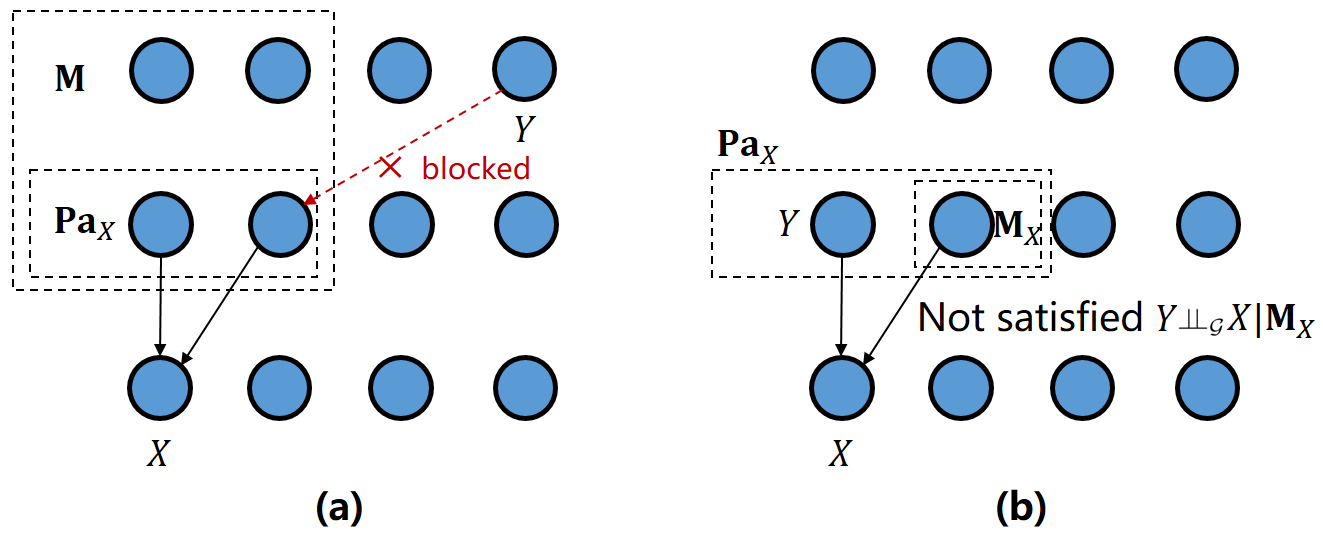}
	\caption{Diagrams for \textbf{Proposition~\ref{nestedM}} and \textbf{Corollary~\ref{paxmx}}. (a) In \textbf{Proposition~\ref{nestedM}}, for any $\mathbf{M}\supseteq \mathbf{Pa}_X$, $\forall Y \in \mathbf{V}\backslash\mathbf{M} \cup \{X\}$, $\mathbf{M}$ blocks any information paths from $Y$ to $X$. (b) In \textbf{Corollary~\ref{paxmx}}, if $\mathbf{Pa}_X \supset \mathbf{M}_X$, $\exists Y\in \mathbf{Pa}_X$, but $Y\notin \mathbf{M}_X$, then $Y\nupmodels_\mathcal{G} X | \mathbf{M}_X$ can be obtained.}
	\label{figure: prop5coro2}
\end{figure}

Thus, $\mathbf{Pa}_X = \mathbf{M}_X$ can be obtained with few proofs.
\begin{corollary}
	$\mathbf{Pa}_X = \mathbf{M}_X$ for variable $X$ in STBN $\mathcal{G}$,
	\label{paxmx}
\end{corollary}
\begin{proof}	
	Let's discuss it in categories. Firstly, we suppose $\mathbf{M}_X$ is the Markov blanket of the variable $X$ according to \textbf{Definition~\ref{mb}}. Then, on the one hand, if $\mathbf{Pa}_X \subseteq \mathbf{M}_X$, $\mathbf{Pa}_X = \mathbf{M}_X$ obviously due to \textbf{Proposition~\ref{nestedM}}, as shown in Fig.~\ref{figure: prop5coro2}(a). Thus, $\mathbf{M}_X$ is the minimal realization of $\mathbf{M}$. On the other hand, if $\mathbf{Pa}_X \supset \mathbf{M}_X$, then $\exists Y$ precedes $X$, $Y\in \mathbf{Pa}_X$, but $Y\notin \mathbf{M}_X$. Then, we obtain $Y\nupmodels_\mathcal{G} X|\mathbf{M}_X$ because there is a direct connection between $Y$ and $X$, and $\mathbf{M}_X$ cannot d-separates them, as shown in Fig.~\ref{figure: prop5coro2}(b). But due to \textbf{Definition~\ref{mb}}, and $\mathbf{M}_X$ is the Markov blanket of $X$, thus we obtain $Y\upmodels_\mathcal{G} X|\mathbf{M}_X$. Thus, this is contradictory. Thus, $\mathbf{Pa}_X = \mathbf{M}_X$ only.
	
\end{proof}
Thus, the variables providing prediction knowledge to $X$ is only up to causal parents $\mathbf{Pa}_X$ in STBN $\mathcal{G}$. Meanwhile, that means $\mathbf{Pa}_X$ can be identified in parallel, because the neighbor variables at the same time provide no any prediction knowledge to $X$. This induces decomposability property of STBN $\mathcal{G}$, as follows:
\begin{proposition}[Decomposability]
	STBN $\mathcal{G}$ at time $t$ can be decomposed into $n$ subgraphs $\mathcal{G}_{i,t}=(\mathbf{V}, \mathbf{E}_{i,t})$ for simultaneous variables $X_{i,t}, i=1, \dots, n$. Then, d-separation in $\mathcal{G}$ is the same as that in $\mathcal{G}_{i,t}$, that is, $\upmodels_\mathcal{G} \iff \upmodels_{\mathcal{G}_{i,t}}$.
	\label{decom}
\end{proposition}
\begin{proof}
	According to \textbf{Corollary~\ref{paxmx}}, the predictability for each target variable $X_{i,t}, i=1, \dots, n$ is only up to their parents $\mathbf{Pa}_i$. And, there must be no connection between these variables at the same time $t$, and no connection backtracking from time $t$ to the past, because of the \textbf{temporal assumption}. Thus, (i) For each variable at time $t$, e.g., $X_{1, t}$, the variables $X_{2, t}, \dots, X_{n, t}$ would not be in the set $\mathbf{Pa}_{1}$, where $\mathbf{Pa}_{1}$ is the causal parents of $X_{1, t}$. (ii) For each variable at time $t$, e.g., any causal pathways pointing to $X_{1, t}$ through $X_{2, t}, \dots, X_{n, t}$ would be blocked. Thus, the causal links regarding $X_{2, t}, \dots, X_{n, t}$ and the other pathways through them can be deleted. Then, the reduced subgraph $\mathcal{G}_{1, t}$ is obtained. It is the same to $X_{2, t}, \dots, X_{n, t}$, and $\mathcal{G}_{2, t}, \dots, \mathcal{G}_{n, t}$ can also be obtained, respectively. Thus, $\mathcal{G}$ can be decomposed into $n$ subgraphs $\mathcal{G}_{i,t}=(\mathbf{V}, \mathbf{E}_{i,t})$. Here, $\mathbf{E}_{i,t}$ is the directed connections from each variable in $\mathbf{Pa}_i$ to $X_i$. That is, $\mathcal{G}$ is the linear combination of all $\mathcal{G}_{i,t}$. Moreover, $\upmodels_\mathcal{G} \iff \upmodels_{\mathcal{G}_{i,t}}$ is trivial to see in Fig.~\ref{decomposition}.  
\end{proof}

\begin{figure}[t!]
	\centering
	\includegraphics[width=\linewidth]{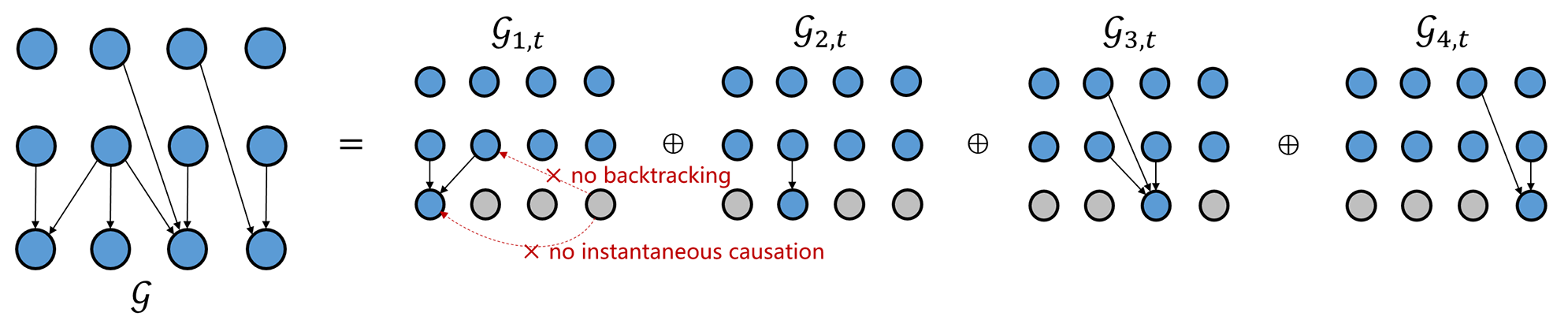}
	\caption{Diagram of decomposing STBN and combining subgraphs. Here, $\bigoplus$ represents linear combination.}
	\label{decomposition}
\end{figure}

Then, the uniqueness property of STBN $\mathcal{G}$ can be finally obtained as follows: 
\begin{proposition}[Uniqueness]
	The network structure of STBN $\mathcal{G}$ is unique.
\end{proposition}
\begin{proof}
	According to the principle of decomposability, we have $n$ subgraphs $\mathcal{G}_{i,t}$ of $\mathcal{G}$. In $\mathcal{G}_{i,t}$, the parents of the target variable $X_{i,t}$ is unique. we can first prove it.
	
	Let $X$ be any target variable $X_{i,t}$. Then, $\mathbf{Pa}_X$ and $\widetilde{\mathbf{Pa}}_X$ are both the parents of $X$ in $\mathcal{G}_{i,t}$ and $\widetilde{\mathcal{G}}_{i,t}$, respectively. $\forall Y \in \widetilde{\mathbf{Pa}}_X$, no variable set $\mathbf{Z}$ excluding $X$ and $Y$ can satisfy $Y\upmodels_{\widetilde{\mathcal{G}}_{i,t}} X | \mathbf{Z}$. Thus, $\mathbf{Pa}_X\backslash \{Y\}$, one possible $\mathbf{Z}$, is also not able to do so. Thus, $Y\nupmodels_{\widetilde{\mathcal{G}}_{i,t}} X | \mathbf{Pa}_X\backslash \{Y\}$, and then $Y\nupmodels_{{\mathcal{G}}_{i,t}} X | \mathbf{Pa}_X\backslash \{Y\}$ because $\upmodels_{\widetilde{\mathcal{G}}_{i,t}} \Rightarrow \upmodels_\mathcal{G} \Rightarrow \upmodels_{\mathcal{G}_{i,t}}$ (see decomposability in \textbf{Proposition~\ref{decom}}). Then, $Y\in \mathbf{Pa}_X$ is obtained by \textbf{Corollary~\ref{inpa}} and the causal equivalence in \textbf{Proposition~\ref{ceq}}. Thus, $\widetilde{\mathbf{Pa}}_X \subseteq \mathbf{Pa}_X$. In a similar way, $\widetilde{\mathbf{Pa}}_X \supseteq \mathbf{Pa}_X$ can be also obtained. Thus, $\mathbf{Pa}_X=\widetilde{\mathbf{Pa}}_X$. Thus, $\mathcal{G}_{i,t} = \widetilde{\mathcal{G}}_{i,t}$. 
	
	Thus, the network structure of STBN $\mathcal{G}$ is unique because all subgraphs $\mathcal{G}_{i,t}$ are unique, and $\mathcal{G}$ is just the linear combination of multiple unique subgraphs. Thus, there is no Markov equivalence issue in STBN.
\end{proof}



\subsection{High-order Causal Entropy Algorithm}







To identify unique network structure of STBN, HCE algorithm is proposed by the transfer entropy under the principle of information path blocking. It is a special case of transfer entropy, thus we call it high-order causal entropy here. The pseudocode of HCE algorithm is shown in \textbf{Algorithm~\ref{hce}}. 

Due to the decomposability property (see \textbf{Proposition~\ref{decom}}), the network can be decomposed into $n$ subgraphs. These subgraphs can be identified independently and combined to be a complete network finally. Thus, HCE algorithm can be implemented by $n$ child processes. Each child process searches causal parents $\mathbf{Pa}_i$ for each variable $X_i\in\{X_1, \dots, X_n\}$. This is in a paralleled way instead of $n$-times sequential iterations.

Then, each child process can be decomposed into two modules by parameters $\tau_{max}, \alpha$ and $\beta$. In the first module, for target variable $X_i$, HCE algorithm calculates the causal entropy for each historical variable and each time delay $\tau=1, \dots, \tau_{max}$. This module saves part of variables from the given full condition set, and finally obtains possible $\mathbf{Pa}_i$. In the second module, HCE algorithm removes spurious causation from $\mathbf{Pa}_i$, and the subgraph is obtained at the end. 

The parameters $\alpha$ and $\beta$ are both used to control the bound of zero causal entropy according to \textbf{Corollary~\ref{inpa}}. If the value of causal entropy is less than $\alpha$ or $\beta$, the value is to approximate to be zero. Thus, the obtained $\mathbf{Pa}_i$ becomes narrower with bigger $\alpha$ and $\beta$. Finally, the network can be obtained by combining all subgraphs. 

Moreover, the time complexity of HCE algorithm is $\mathcal{O}(n^3\tau_{max})$. The main computation is up to searching causal parents for each variable at time $t$ and estimating causal entropy for each possible edge. The time to search causal parents is $\mathcal{O}(n\tau_{max})$ because all historical variables are to be traversed. Then, note that HCE algorithm can be in paralleled, thus this time is the same for each variable, and the time to combine all subgraphs is $\mathcal{O}(1)$, that can be ignored. Meanwhile, KSG estimator~\cite{KSG2004PRE} is adopted to measure the causal entropy with time complexity $\mathcal{O}(n^2)$ in worst case. But it can be accelerated by using KD-tree neighbor search.

\begin{algorithm}
	\LinesNumbered
	\KwIn{Observation variables $\mathbf{V}=\{{X}_1, \dots, {X}_n\}$ and maximum time lag $\tau_{max}$.}
	\KwOut{An STBN $\mathcal{G}$.}
	
	\For{$i=1, \dots, n$ in parallel}{
		Initialize causal parent set $\mathbf{Pa}_i=\varnothing$; \\
		\For{$\tau=1, \dots, \tau_{max}$}{
			Let condition set $\mathbf{Z}=\{X_{1,t-\tau}, \dots, X_{n,t-\tau}\}$; \\
			\For{$j=1, \dots, n$}{
				$\mathcal{T}_{X_{j,t-\tau}\rightarrow X_{i,t}|\mathbf{Z}\backslash \{X_{j,t-\tau}\}} = H(X_{i,t}|\mathbf{Z}\backslash \{X_{j,t-\tau}\}) - H(X_{i,t}|\mathbf{Z})$; \\
				\eIf{$\mathcal{T}_{X_{j,t-\tau}\rightarrow X_{i,t}|\mathbf{Z}\backslash \{X_{j,t-\tau}\}} > \alpha$}
				{
					append $X_{j,t-\tau}$ into $\mathbf{Pa}_i$;
				}{
					remove $X_{j,t-\tau}$ from $\mathbf{Z}$;
				}
			}
		}
		
		\For{{\rm each} $Y\in \mathbf{Pa}_i$}{
			$\mathcal{T}_{Y\rightarrow X_{i,t}|\mathbf{Pa}_i\backslash \{Y\}}=H(X_{i,t}|\mathbf{Pa}_i\backslash \{Y\}) - H(X_{i,t}|\mathbf{Pa}_i)$; \\
			\If{$\mathcal{T}_{Y\rightarrow X_{i,t}|\mathbf{Pa}_i\backslash \{Y\}} < \beta$}
			{
				remove $Y$ from $\mathbf{Pa}_i$;
			}
		}
	}
	Construct the network $\mathcal{G}$ with $\mathbf{Pa}_1, \dots, \mathbf{Pa}_n$;
	\caption{High-order Causal Entropy Algorithm}
	\label{hce}
\end{algorithm}

\section{Experiments}

\subsection{Synthetic Data}

In the experiments, the performance of HCE algorithm is tested on synthetic data generated from STBNs. For example, as shown in Fig.~\ref{ucnd}, $f_1, f_2$ are defined as $f_1(x) = x$ and $f_2(x) = x + 5x^2 e^{-\frac{x^2}{20}}$. The state of variable at time $t$ is certainly calculated by the historical states of all variables, such that $x_{1,t} = 0.2f_2(x_{1,t-1}) + 0.3f_1(x_{env,t-1}) + \mathcal{N}(0, 1)$, and $\mathcal{N}(0, 1)$ is noise conforming to standard normal distribution. Moreover, the environment variable $X_{env}$ is the common parent for variable $X_1$ and variable $X_4$, and $X_{env}$ has no causal parents.

Then, to generate trending time-series data, the state of environment variable $X_{env}$ is set to be a uniformly increasing sequence as shown in Fig.~\ref{timeseries}. Clearly, the trajectory of $X_{env}$ is in the range of $1.3\sim8.0$ within 2000 time steps. Meanwhile, the other observation variables are also presented in Fig.~\ref{timeseries}. They all fluctuate in different and limited ranges, and have no property of periodicity. As shown in Fig.~\ref{ds}, it is presented the distributions of four observation variables within first 1000 time steps and within later 1000 time steps respectively. Clearly, dataset shift occurs, so that regression-based models are not to work.  

Moreover, all data are randomly generated from STBNs like Fig.~\ref{ucnd}, but we do not list them all here. In practice, their network size $n$ and the number of edges are kept in a certain proportion, and the number of edges is at least twice the network size. All the edges are set at random to guarantee the randomization of the experiments. Moreover, the sample size $k$ is always fixed to be 2000.

\begin{figure}[t!]
	\centering
	\includegraphics[width=0.8\linewidth]{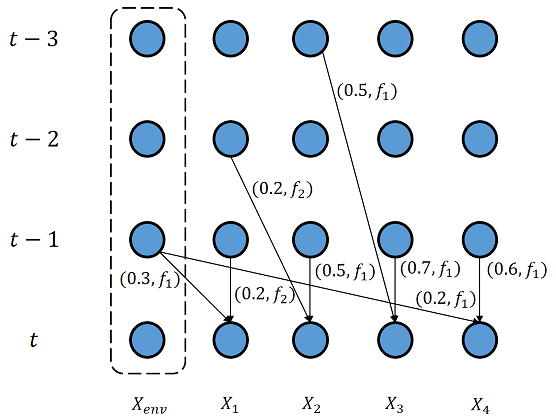}
	\caption{An example of STBN.}
	\label{ucnd}
\end{figure}

\begin{figure}[t!]
	\centering
	\includegraphics[width=\linewidth]{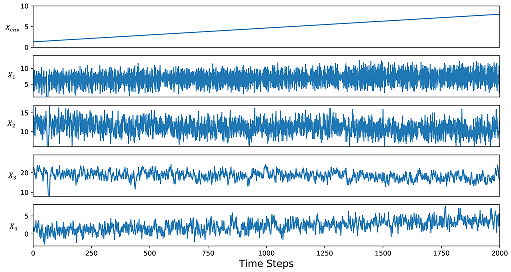}
	\caption{Trending time series of five variables.}
	\label{timeseries}
\end{figure}

\begin{figure}[t!]
	\centering
	\includegraphics[width=\linewidth]{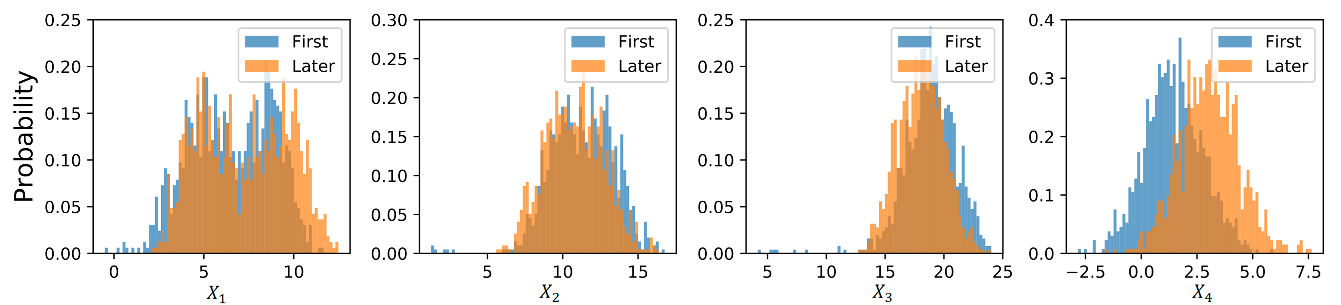}
	\caption{Dataset shift of the four observation variables.}
	\label{ds}
\end{figure}

\subsection{Evaluation Metrics}

In the experiments, two indices are adopted to measure the performance of model, including true positive rate (TPR) and false positive rate (FPR). They are defined as 
\begin{equation}
	\begin{aligned}
		TPR &= \frac{true~positive}{true~positive + false~negative}, \\
		FPR &= \frac{false~positive}{false~positive + true~negative}.
	\end{aligned}
\end{equation}

\subsection{Results}


To conduct comparative experiments, HCE algorithm is compared with GC test~\cite{GC1969}, convergent cross mapping (CCM) algorithm~\cite{CCM2012}, PCMCI algorithm~\cite{PCMCI2019}, and DYNOTEARS algorithm~\cite{dynotears2020PRML}. Among them, GC test is a classic algorithm to identify time-series causality by vector regression model. CCM algorithm is proposed to detect causality with high-order time delay in complex ecosystem based on the tool of nonlinear state space reconstruction. PCMCI algorithm is proposed to identify Full Time Graph from large time-series dataset. DYNOTEARS algorithm is a score-based algorithm that minimizes a penalized loss subject to an acyclicity constraint. But note that, it is based on linear modeling, and it allows instantaneous causation.

In the experiments, the truth of STBN is hidden firstly, and then the network structure is identified by these algorithms. The TPR and FPR are both calculated after running these algorithms on the synthetic data. The parameters of HCE are set as that $\tau_{max}=5, \alpha=0.01$ and $\beta=0.02$. 

As shown in Fig.~\ref{roc}, the receiver operating characteristic (ROC) curve is drawn with TPRs and FPRs on time series generated by STBN in Fig.~\ref{ucnd}. The ROC curve of HCE covers the curves generated from the other algorithms. That is, the area under curve (AUC) of HCE is higher than that of other algorithms. Thus, HCE algorithm has better generalization.

Moreover, the performance of HCE algorithm is tested on dataset with bigger network size. As shown in Fig.~\ref{acc}, it is clear that HCE algorithm performs generally better than the other baseline algorithms with increasing network size. Note that all results for boxplot are tested on dataset generated from STBNs with different network sizes and different network structures. To eliminate test errors, the data generation processes are implemented at least 20 times. For all algorithms, the FPRs are fixed to be below 0.1 as possible, so that the TPRs are easy to compare.

\setcounter{figure}{8}
\begin{figure*}[t!]
	\centering
	\includegraphics[width=\linewidth]{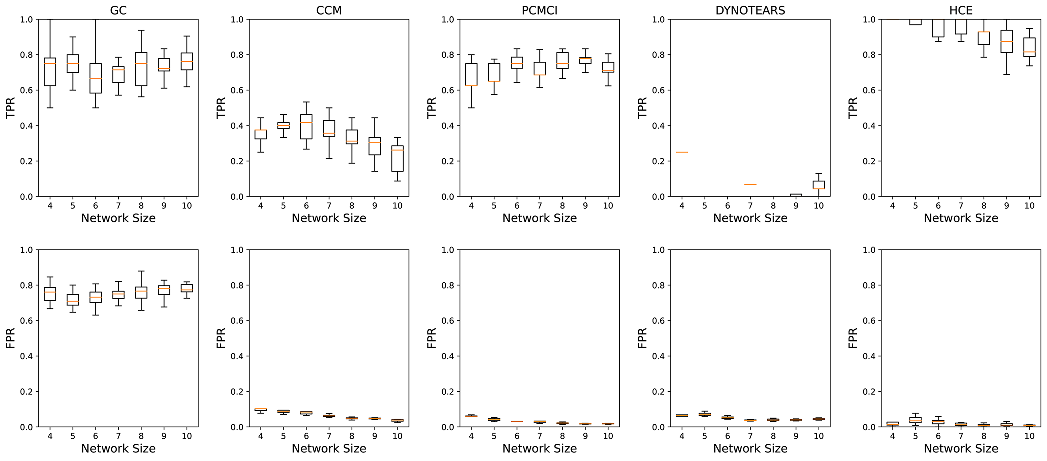}
	\caption{The comparison to identification accuracy of HCE algorithm and other baseline algorithms.}
	\label{acc}
\end{figure*}

\setcounter{figure}{7}
\begin{figure}[htbp]
	\centering
	\includegraphics[width=0.92\linewidth]{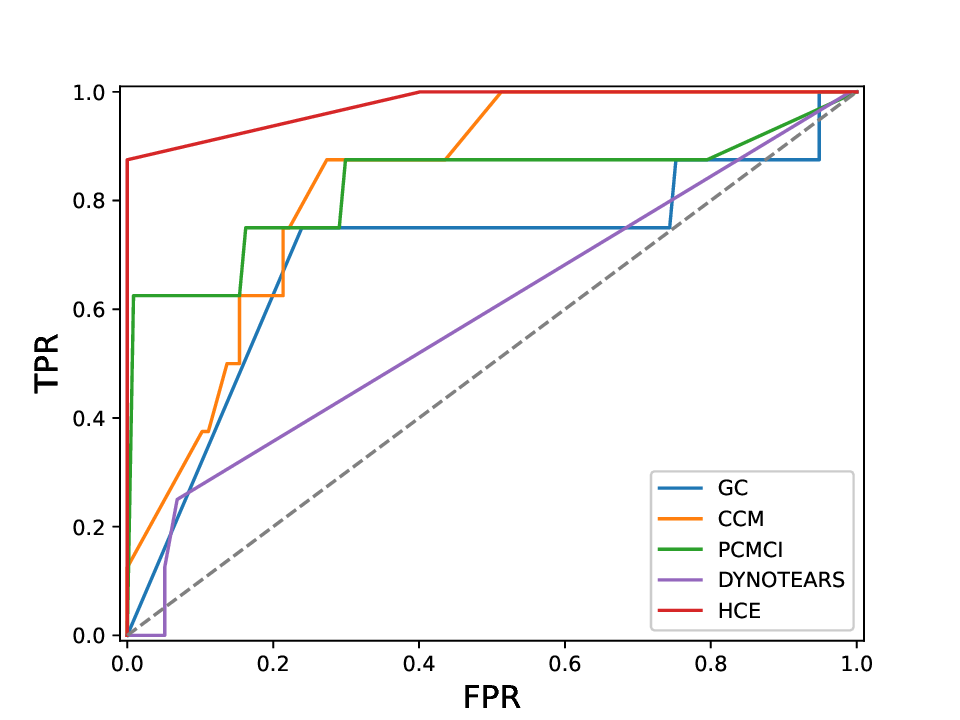}
	\caption{The ROC curve of HCE and other algorithms.}
	\label{roc}
\end{figure}

As shown in Fig.~\ref{acc}, the TPRs of HCE algorithm gradually decrease with the increase of network size, but it is still higher than that of PCMCI algorithm with the same FPRs. By contrast, GC and CCM algorithms have lower TPRs and larger variances at each running time. Moreover, DYNOTEARS algorithm poorly performs on the TPRs when the FPRs are fixed, with increasing network size. Also, it poorly performs on the ROC curve, as shown in Fig.~\ref{roc}. The reasons are as follows: (i) The modeling of DYNOTEARS is linear, but the data generation is nonlinear in the experiments of this work. (ii) DYNOTEARS considers the instantaneous effects between variables, that induces the huge error to detect many spurious instantaneous causal links, as shown in Fig.~\ref{figure: nonzero}. Clearly, there are many non-zero causal links when time lag $\tau = 0$, but actually, these links are spurious. Thus, if given some observation variables, and the variable number is less than 10, HCE algorithm can identify the true structure of STBN, even though the data distribution of time series is shifted with time.

\setcounter{figure}{9}
\begin{figure}[htbp]
	\centering
	\includegraphics[width=0.92\linewidth]{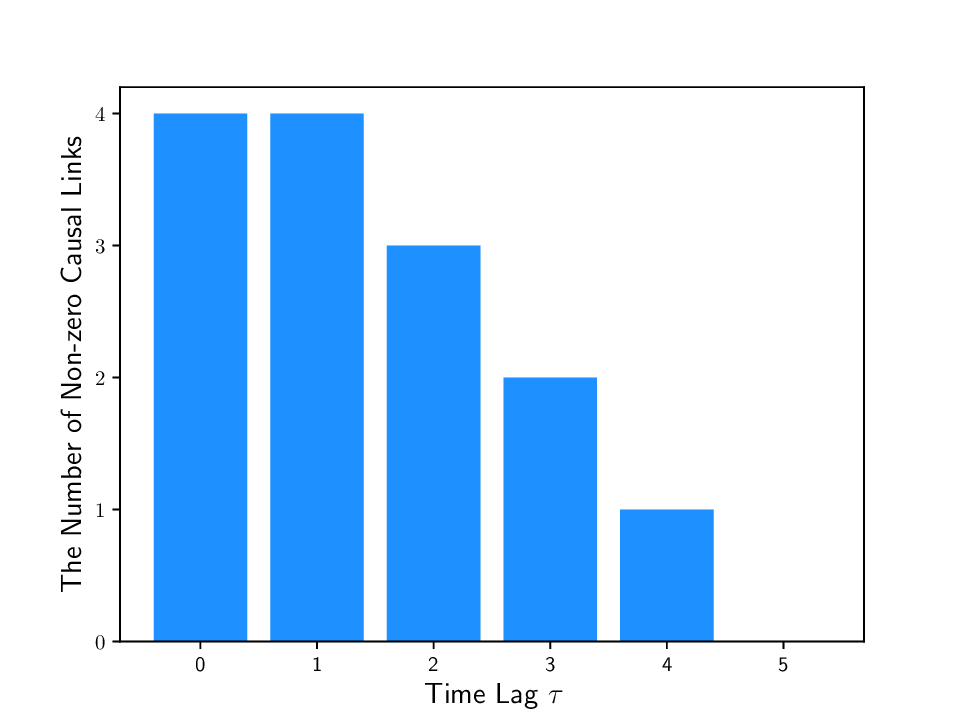}
	\caption{The distribution of non-zero causal links on an example of DYNOTEARS.}
	\label{figure: nonzero}
\end{figure}

\section{Conclusion and Future Work}



In this work, STBN is proposed to model the spatial-temporal causality by introducing the principle of information transfer. From this perspective, information path blocking explains the disappearance of $X\rightarrow Z \rightarrow Y$ and $X\leftarrow Z \leftarrow Y$ in STBN, and further induces the uniqueness of network structure. Based on this, HCE algorithm is designed to identify the unique network structure by measuring transfer entropy with information path blocking. Sufficient experiments are conducted to support the main conclusions regarding the performance of HCE algorithm, compared with GC test, CCM algorithm, PCMCI algorithm, and DYNOTEARS algorithm. The results show the state-of-the-art identification accuracy of HCE algorithm on dataset with time-varying shift and increasing network size.

In future works, we will focus on the goal of causal inference based on the proposed STBN model, that is, to further build a framework for inferring spatial-temporal counterfactuals. One possible approach is to use STBN to initialize the network in interactive dynamics~\cite{gao2022autonomous}. The interactive dynamics can be sparsely identified with the network. And then, the network structure can be intervened, and the counterfactuals can be solved by resetting the initialization values. Another solution to realize this can be found in paper~\cite{ness2019integrating}, but we find it only generalized on one-order Markov process, which is not universal. A non-temporal but important case can be found in paper~\cite{pawlowski2020deep}. If use the deep generative model to model the network, the supported network size would be larger, but the explainability (or interpretability) would be reduced.

Moreover, we find it different in the intervening operations between STBNs and the non-temporal Bayesian networks. As stated in~\cite{Aprimer2016}, if a non-temporal Bayesian network is intervened, the network structure would be changed, and the intervened probability distributions are to be recalculated though the backdoor criterion and front-door criterion. But it is different in STBNs, an STBN models the momentary causality at time $t$, thus an intervention on STBN would only be defined at that time, but not during a period of times. Thus, to describe the momentary interventions accurately, a group of interventions at different timestamps are needed. This makes the intervening operations to STBNs more complex, and more flexible meanwhile. This is also one focus of the future works.

\bibliographystyle{IEEEtran}
\bibliography{ieeetai}

\vspace{0mm}
\begin{IEEEbiography}[{\includegraphics[width=1in,height=1.25in,clip,keepaspectratio]{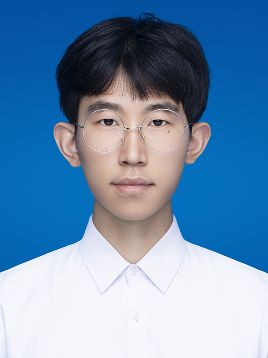}}]{Mingyu Kang}
	received M.S. degree of computer technology from Southeast University in China in 2021. He is currently a Ph.D candidate with the Jiangsu Key Laboratory of Networked Collective Intelligence, School of Cyber Science and Engineering, Southeast University, Nanjing, China. His research interests include complex system, complex network, deep learning and causal inference.
\end{IEEEbiography}

\vspace{0mm}
\begin{IEEEbiography}[{\includegraphics[width=1in,height=1.25in,clip,keepaspectratio]{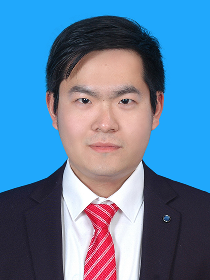}}]{Duxin Chen}
	received the B.S. degree in automatic control and the Ph.D. degree in control science and engineering from the Huazhong University of Science and Technology,Wuhan,China, in 2013 and 2018, respectively. He is currently an associate professor with the School of Mathematics, and vice director of the Jiangsu Provincial Scientific Research Center of Applied Mathematics, Southeast University, Nanjing, China. He is also one of the direction leaders of the Huawei-SEUJoint Innovation Lab of Networked Collective Intelligence. His research interests include causal inference, prediction/generation, and system identification techniques for complex networks and systems science, artificial intelligence related theory and applications.
\end{IEEEbiography}

\vspace{0mm}
\begin{IEEEbiography}[{\includegraphics[width=1in,height=1.25in,clip,keepaspectratio]{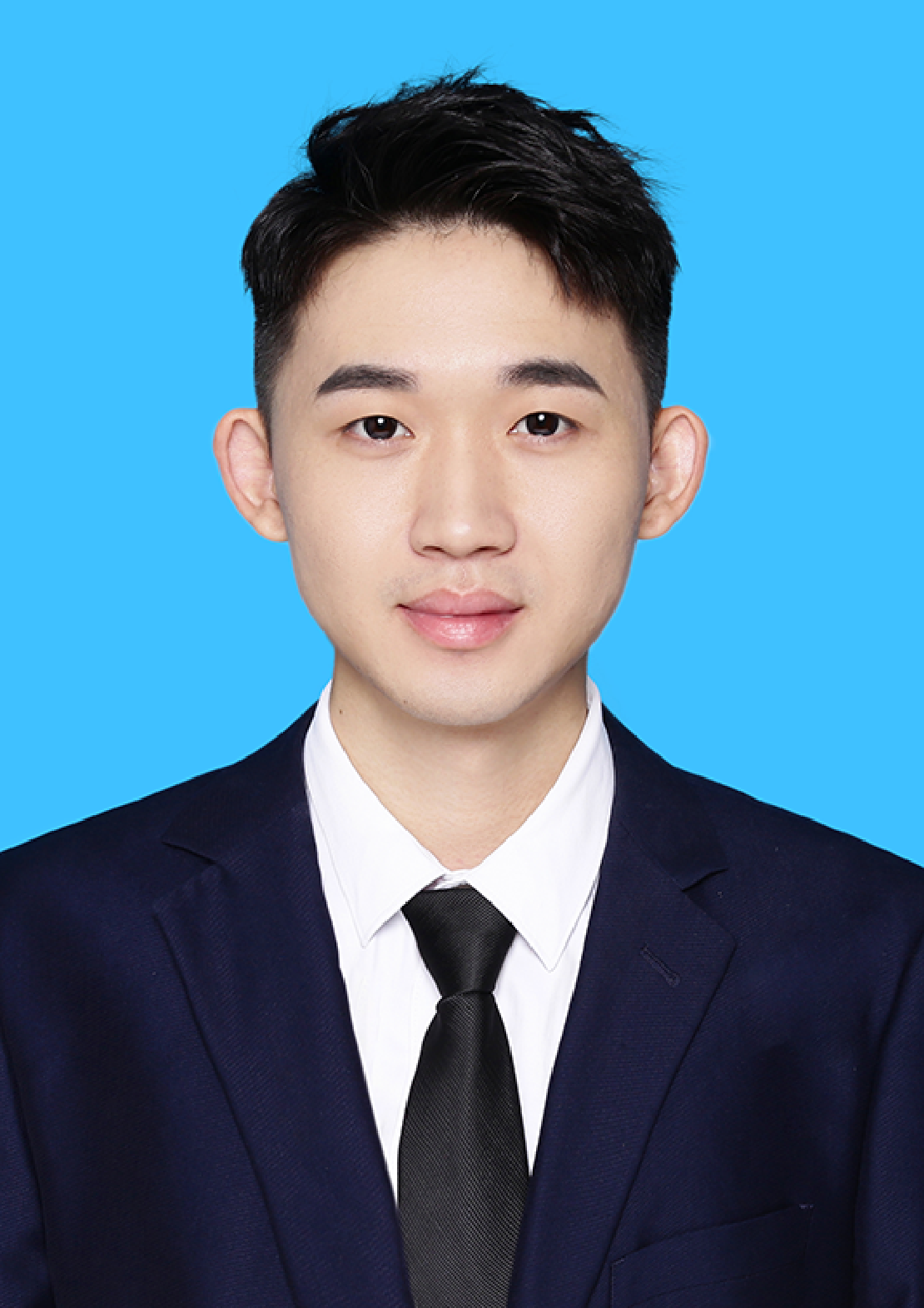}}]{Ning Meng}
	received M.S. degree of cyberspace security from Southeast University in China in 2023. He is currently a senior engineer in the R\&D department of Volkswagen China Technology Company (VCTC).
\end{IEEEbiography}

\vspace{0mm}
\begin{IEEEbiography}[{\includegraphics[width=1in,height=1.25in,clip,keepaspectratio]{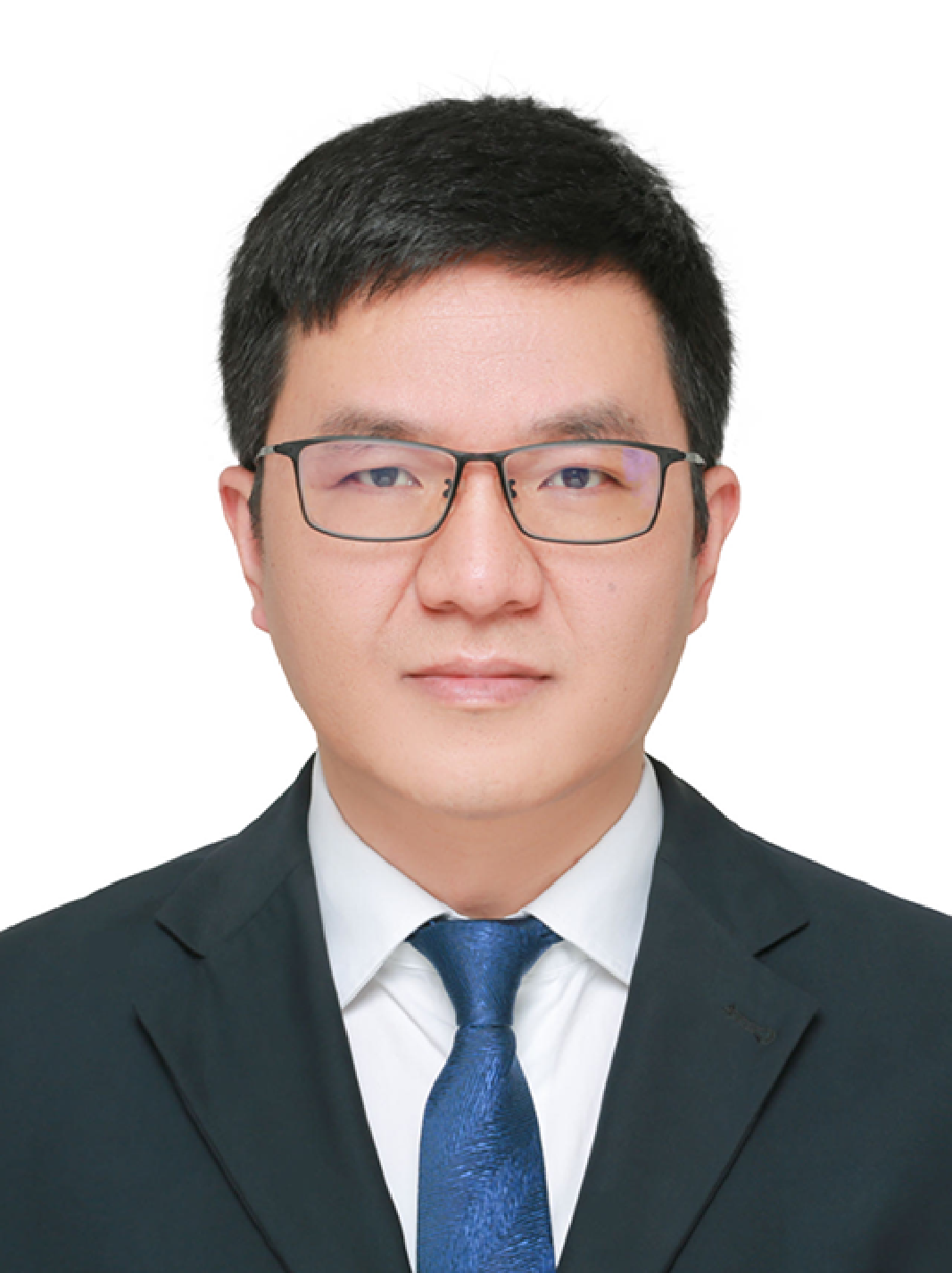}}]{Gang Yan}
	received the B.S. and Ph.D. degrees in Electronic Engineering from University of Science and Technology of China, Hefei, China, in 2005 and 2010, respectively. He was with Temasek Laboratories at National University of Singapore, Singapore, and Network Science Institute at Northeastern University, Boston, USA as Research Scientist and Research Fellow in 2010-2013 and 2013-2016, respectively. He joined Tongji University, Shanghai, China in 2017 and is currently a Distinguished Professor. 
	
	His research lies at the interface between complex systems and artificial intelligence and his research results have been published in Nature, Nature Physics, Nature Computational Science, Nature Communications, NeurIPS, AAAI, etc. He was the recipient of Distinguished Young Scholar from National Science Foundation of China in 2022, and the Excellent Editor of the IEEE Transactions on Network Science and Engineering in 2021.
\end{IEEEbiography}

\vspace{0mm}
\begin{IEEEbiography}[{\includegraphics[width=1in,height=1.25in,clip,keepaspectratio]{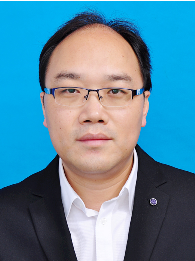}}]{Wenwu Yu} (S'07-M'12-SM'15) received the B.Sc. degree in information and computing science and M.Sc. degree in applied mathematics from the Department of Mathematics, Southeast University, Nanjing, China, in 2004 and 2007, respectively, and the Ph.D. degree from the Department of Electronic Engineering, City University of Hong Kong, Hong Kong, China, in 2010. Currently, he is the Dean in the School of Mathematics, the Deputy Associate Director of National Center of Applied Mathematics in Southeast University of Jiangsu, the Deputy Director of Jiangsu Provincial Scientific Research Center of Applied Mathematics, Deputy Associate Director of Jiangsu Provincial Key Laboratory of Networked Collective Intelligence, and a Full Professor with the Endowed Chair Honor in Southeast University, China.
	
Dr. Yu held several visiting positions in Australia, China, Germany, Italy, the Netherlands, and the USA. His research interests include multi-agent systems, complex networks and systems, disturbance control, distributed optimization, machine learning, game theory, cyberspace security, smart grids, intelligent transportation systems, big-data analysis, etc. 

Dr. Yu serves as an Editorial Board Member of several flag journals, including IEEE Transactions on Circuits and Systems II, IEEE Trans. Industrial Cyber-Physical Systems, IEEE Transactions on Industrial Informatics, IEEE Transactions on Systems, Man, and Cybernetics: Systems, Science China Information Sciences, Science China Technological Sciences, etc. 

He was listed by Clarivate Analytics/Thomson Reuters Highly Cited Researchers in Engineering in 2014-2023. He publishes about 100 IEEE Transactions journal papers with more than 20,000 citations. Moreover, Dr. Yu is also the recipient of the Second Prize of State Natural Science Award of China in 2016. He is also the Cheung Kong Scholars Programmer of Ministry of Education of China (Artificial Intelligence).
\end{IEEEbiography}

\end{document}